\documentclass[letterpaper]{article} 
\usepackage{aaai24}  
\usepackage{times}  
\usepackage{helvet}  
\usepackage{courier}  
\usepackage[hyphens]{url}  
\usepackage{graphicx} 
\usepackage{natbib}  
\usepackage{caption} 
\usepackage{algorithm}
\usepackage{algorithmic}

\usepackage{newfloat}
\usepackage{listings}

\usepackage{soul}
\usepackage{subfigure}
\usepackage{multirow}
\usepackage{booktabs}
\usepackage{amsthm,amsmath,amssymb}
\usepackage{mathrsfs}
\usepackage{arydshln}
\usepackage{eqparbox}

\newtheorem{myTheo}{Theorem}
\newtheorem{myDef}{Definition}
\newtheorem{myTheoApp}{Theorem}

\usepackage[inline]{enumitem}

\usepackage[symbol]{footmisc}
\usepackage{caption}

\usepackage{graphicx}
\usepackage{placeins}

\usepackage{xr}
\makeatletter
\newcommand*{\addFileDependency}[1]{
  \typeout{(#1)}
  \@addtofilelist{#1}
  \IfFileExists{#1}{}{\typeout{No file #1.}}
}
\makeatother

\DeclareCaptionStyle{ruled}{labelfont=normalfont,labelsep=colon,strut=off} 
\floatstyle{ruled}
\newfloat{listing}{tb}{lst}{}
\floatname{listing}{Listing}
\title{Anchoring Path for Inductive Relation Prediction in Knowledge Graphs}

\author{
    Zhixiang Su\textsuperscript{\rm 1,2,3},
    Di Wang\textsuperscript{\rm 2,3},
    Chunyan Miao\textsuperscript{\rm 1,2,3,4}
    Lizhen Cui\textsuperscript{\rm 4,5}
}
\affiliations{
    \textsuperscript{\rm 1}School of Computer Science and Engineering, Nanyang Technological University (NTU), Singapore\\
    \textsuperscript{\rm 2}Joint NTU-UBC Research Centre of Excellence in Active Living for the Elderly (LILY), NTU, Singapore\\
    \textsuperscript{\rm 3}WeBank-NTU Joint Research Institute on Fintech, NTU, Singapore\\
    \textsuperscript{\rm 4}SDU-NTU Centre for Artificial Intelligence Research (C-FAIR), Shandong University (SDU), China\\
    \textsuperscript{\rm 5}School of Software, SDU, China\\
    \{zhixiang002, wangdi, ascymiao\}@ntu.edu.sg, clz@sdu.edu.cn

}

\usepackage{bibentry}

\begin{document}
\maketitle

\begin{abstract}
      Aiming to accurately predict missing edges representing relations between entities, which are pervasive in real-world Knowledge Graphs (KGs), relation prediction plays a critical role in enhancing the comprehensiveness and utility of KGs. Recent research focuses on path-based methods due to their inductive and explainable properties. However, these methods face a great challenge when lots of reasoning paths do not form Closed Paths (CPs) in the KG. To address this challenge, we propose Anchoring Path Sentence Transformer (APST) by introducing Anchoring Paths (APs) to alleviate the reliance of CPs. Specifically, we develop a search-based description retrieval method to enrich entity descriptions and an assessment mechanism to evaluate the rationality of APs. APST takes both APs and CPs as the inputs of a unified Sentence Transformer architecture, enabling comprehensive predictions and high-quality explanations. We evaluate APST on three public datasets and achieve SOTA performance in 30 of 36 transductive, inductive, and few-shot experimental settings.
\end{abstract}

\section{Introduction}\label{section_introduction}
    As a structured representation of knowledge (e.g., real-world facts), Knowledge Graphs (KGs) serve as the basis for various downstream tasks requiring reasoning, such as question answering. KGs generally consist of a collection of triplets, denoted as $G(E_G, R_G)=\{(h_i,r_i,t_i)|i=1,2,3,..,m\}$, where entities (head and tails, $h_i,t_i \in E_G$) and relations ($r_i \in R_G$) are represented as nodes and edges, respectively. However, real-world KGs are often incomplete, making relation prediction a necessary step to validate missing relations and subsequently enhance the comprehensiveness and reasoning ability of KGs. Given a query triplet $(h,r,t)$ with $h$ (or $t$) masked, relation prediction aims to predict the masked entity, utilizing the incomplete knowledge stored in the KG.

    Relation prediction approaches can generally be classified into three different types: embedding-based, Graph Neural Network (GNN)-based, and path-based methods. Recent research has focused on path-based methods (e.g., KRST~\cite{KRST}) due to their notable advantages as follows:

    \noindent \textbf{Inductiveness:} Unlike traditional approaches that focus on transductive situations (where the entity set is static and all entities are observed during training), real-world dynamic graphs are constantly evolving. This necessitates models to possess inductive capabilities to deal with unseen entities. Most embedding-based and GNN-based methods generate embeddings, which are shown as difficult to represent unseen entities. Comparatively, path-based methods are well suited to handle the inductive cases by making predictions based on a chain of relations between the head and tail entities. Such a chain of relations is usually regarded as the \textbf{supporting evidence}  in this reasoning process. Once the path information is learned, similar chains of relations can be extrapolated for predictions of unseen entities.

    \noindent \textbf{Explainability:} Explainability holds significant importance in relation prediction, encompassing a model's ability to justify its predictive decisions by illustrating the essential reasoning process.  Embedding-based and GNN-based methods have difficulties in this aspect because their encoded information is incomprehensible to humans. Comparatively, path-based methods inherently offer explainability, because the paths themselves can serve as the supporting evidence for predictions with explanations.

    Recent approaches (e.g., BERTRL~\cite{BERTRL} and KRST~\cite{KRST}) combined path-based methods with pre-trained language models (PLMs). PLMs take text descriptions as input, offering valuable auxiliary information for encoding paths and triplets. These text descriptions may significantly enrich the semantic meanings of entities and relations. By leveraging text descriptions, these recent approaches achieved state-of-the-art (SOTA) performance on the most widely benchmarked datasets.

    Despite the great success of path-based methods with PLMs in achieving inductive explainable relation prediction, their performance is still limited. Specifically, path-based methods make predictions by leveraging Closed Paths (CPs), where paths are connected to both query head and query tail entities. Therefore, when the given KGs are highly incomplete, i.e., CPs between many head and tail entities are absent, the performance of the existing approaches may be greatly affected. On the other hand, a large portion of supporting evidence is not explicitly provided in the form of CPs. For example, as shown in Figure~\ref{fig1}, consider the query triplet $(\textit{Morgan},\textit{Profession},\textit{Actor})$ with $\textit{Actor}$ masked, a piece of evidence that aligns with human commonsense knowledge is shown in Path 1 as follows:
        \begin{equation*}
        \small
        \label{AP1}
        \textit{Path 1:\:}
            \textit{Morgan} \xrightarrow[]{\textit{Perform}} \textit{TheShawshankRedemption} \xrightarrow[]{\textit{Genre}^{\textit{-1}}}
                \textit{Drama}
        \end{equation*}
    The fact that $\textit{Morgan}$ $\textit{(Freeman)}$ performs in a $\textit{Drama}$ implies with a high probability that he is an $\textit{Actor}$. This evidence connects to the query head ($\textit{Morgan}$) but does not connect to the tail ($\textit{Actor}$) as shown in Figure~\ref{fig1}. Capable of capturing this kind of evidence will greatly enhance the reasoning ability of path-based methods. However, in the realm of KG, researchers are still exploring systematic approaches to effectively capture such commonsense knowledge.

    \begin{figure}[!t]
        \centering
        \includegraphics[scale=0.7]{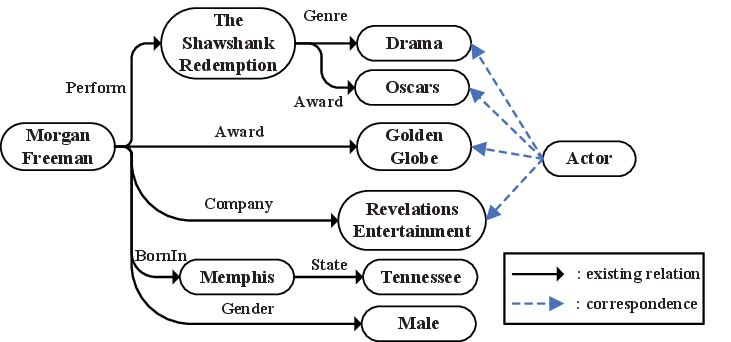}
        \caption{An incomplete KG with paths easily deducible by applying human commonsense knowledge.}\label{fig1}
    \end{figure}

    To further improve the performance of path-based methods with PLMs by incorporating the aforementioned deducible evidence for relation prediction, we propose \textbf{Anchoring Path Sentence Transformer (APST)}, which comprises the following key innovations:

    \noindent \textbf{Enriched text description:} We develop a search-based description enrichment approach, aiming to systematically capture relevant, more detailed descriptions from designated public websites.

    \noindent \textbf{Anchoring Path (AP)  extraction and filtering:} To handle the situations when CPs are absent and alternatively provide multiple pieces of supporting evidence, we define AP to identify reliable paths that connect to only the head or tail entity. APs are easy to obtain, while most of them are logically irrelevant w.r.t the given query triplet. Therefore, we propose AP accuracy and AP recall to assess the rationality of APs and filter out the logically irrelevant ones.  The remaining APs are used for relation prediction, well preserving the inductive and explainable properties of our path-based method.

    \noindent \textbf{Unified Sentence Transformer architecture:} We fine-tune a unified Sentence Transformer to encode inputs comprising both APs and CPs and subsequently make comprehensive relation predictions with intuitive explanations.

    Our key contributions in this work are as follows:

        (\romannumeral1)~We define Anchoring Path (AP) and two AP evaluation metrics, namely AP accuracy and AP recall, to capture and assess supporting evidence other than CPs in KGs.

        (\romannumeral2)~We propose a unified Sentence Transformer named APST, which takes inputs of both APs and CPs with the complement of external textual knowledge captured from public websites.

        (\romannumeral3)~We evaluate the performance of APST on three public datasets. APST achieves SOTA performance in 30 of 36 transductive, inductive, and few-shot experimental settings.

\section{Related Work}\label{section_related_work}
    In this section, we review the three main relation prediction approaches, namely embedding-based, GNN-based, and path-based methods.

    \noindent \textbf{Embedding-based methods} originate from Word2Vec~\cite{word2vec} and represent entities and relations as dense, low-dimensional vectors in a continuous vector space. By associating entities and relations with these embeddings, embedding-based methods measure the distance between entities and relations, facilitating relation predictions based on proximity. Such methods (e.g., TransE~\cite{TransE}, TransH~\cite{TransH}, TransD~\cite{TransD}, TransR~\cite{TransR}, ComplEx~\cite{ComplEx}, ConvE~\cite{ConvE}, and TuckER~\cite{Tucker}) learn embeddings by optimizing certain scoring functions, which are used to capture the relationships between entities. Embedding-based methods offer the advantage of efficient computation and the ability to capture complex relationships between entities in a compact representation. However, they struggle with the generalization to unseen entities and relations. Additionally, these methods are often not explainable because the low-dimensional vectors are not comprehensible to humans.

    \noindent \textbf{GNN-based methods} leverage GNNs to encode the structural information and propagate information through the graph. GNNs operate by aggregating and updating entity representations based on the features of their neighboring entities. By iteratively propagating information through the graph, GNNs can capture complex dependencies and higher-order relationships in the KG. Notable GNN-based models for relation prediction include R-GCN~\cite{RGCN}, CompGCN~\cite{CompGCN}, and GRAIL~\cite{GRAIL}. However, recent studies~\cite{li2022graph,zhang2022rethinking} explored the effective components of GNN-based methods in KGs and found that aggregation actually does not significantly affect the model performance. Consequently, GNN-based explanations relying on neighboring weights may not be always convincing. Moreover, while several GNN-based methods (e.g., GRAIL) are inductive, the majority of GNN-based methods still struggle with inductive cases.

    \noindent \textbf{Path-based methods} leverage CPs in KGs to infer relationships between entities. These methods identify informative paths in CPs connecting two entities and subsequently predict missing relations. Path-based methods often employ path-finding algorithms, such as random walk or breadth-first search (BFS), to explore the KG and identify relevant CPs. Notable path-based methods include DeepPath~\cite{DeepPath}, MINERVA~\cite{MINERVA}, BERTRL~\cite{BERTRL}, and KRST~\cite{KRST}. Compared with the other two types of methods, the distinct advantage of path-based ones is the explicit incorporation of the graph structure with rich information encoded in the paths.

    \noindent \textbf{Methods with PLMs:}
    Leveraging PLMs (e.g., BERT~\cite{BERT}, GPT~\cite{GPT}, and RoBERTa~\cite{Roberta}) for text-related tasks has shown remarkable success in Natural Language Processing (NLP). Recent methods also applied PLMs to enhance the performance of KG-related tasks.
    Such methods (e.g., KG-BERT~\cite{KGBERT}, BERTRL~\cite{BERTRL}, and KRST~\cite{KRST}) utilize the contextualized embeddings generated by PLMs. These embeddings capture the semantics of entities and relations from their text descriptions within sentences or text corpora. By encoding KG information from text descriptions, embeddings can be obtained by feeding triplets or CPs to the PLM and subsequently used as input features for relation prediction.
    Path-based methods with PLMs offer the advantage of capturing both the structural and semantic information. Additionally, the utilization of PLMs improves the generalization ability of path-based methods, allowing them to better comprehend the descriptions of previously unseen entities.

\section{Preliminary}\label{section_preliminary}
   Before introducing APST, we define the inductive relation prediction. Then we introduce Closed Path (CP), which is commonly adopted by path-based methods.

   \begin{myDef}[Inductive Relation Prediction]
    \label{def_inductive_relation_prediction}
    Given a training graph $G_{\textit{train}}(E_{\textit{train}},R_{\textit{train}})$, a testing graph $G_{\textit{test}}(E_{\textit{test}},R_{\textit{test}})$, and a set of query triplets $Q(E_{\textit{query}},R_{\textit{query}})=\{(h_q,r_q,t_q)|h_q, t_q \in E_{\textit{query}}, r_q \in R_{\textit{query}}\}$, a relation prediction is inductive if:
    \begin{itemize}
        \item $E_{\textit{train}} \cap E_{\textit{test}}= \emptyset, E_{\textit{query}} \subseteq E_{\textit{test}}$,
        \item $R_{\textit{test}} \subseteq R_{\textit{train}}, R_{\textit{query}} \subseteq R_{\textit{train}}$.
    \end{itemize}
    \end{myDef}

     Inductive relation prediction requires the model to handle unseen entities. To evaluate the model's inductive ability, we adopt the same settings used in \cite{GRAIL, BERTRL, KRST}. Specifically, during training, only triplets from the training graph $G_{\textit{train}}$ are observed, while the testing graph $G_{\textit{test}}$ is only used to extract supporting evidence for query triplets during testing. Furthermore, to prevent the model from relying on transductive evidence, we set the strictest condition in Definition~\ref{def_inductive_relation_prediction}. Although real-world training and testing graphs may contain overlapping entities, we require $E_{\textit{train}}$ and $E_{\textit{test}}$ have zero overlap to avoid the model relying on transductive evidence.

    Because both testing and query entities are unseen during training, most path-based inductive methods use CPs in KGs for predictions and explanations. For a given query triplet $(h_q, r_q, t_q)$, triplets with the same relation $r_q$ that appear in the training graph can be represented as follows:
    \begin{equation*}
        G_{\textit{train}}(r_q)=\{(h,r_q,t) | (h,r_q,t)\in G_{\textit{train}}\}.
    \end{equation*}
    Among triplets in $G_{\textit{train}}(r_q)$, path-based methods identify paths between heads and tails in the following manner:
    \begin{equation*}
        h \xrightarrow[]{r_1} e_1 \xrightarrow[]{r_2} e_2 \xrightarrow[]{r_3} ... \xrightarrow[]{r_{n-1}} e_{n-1} \xrightarrow[]{r_n} t.
    \end{equation*}
    By disregarding the entities, certain chains of relations commonly appear in the paths. These paths are known as Closed Paths (CPs), which are defined as follows:
    \begin{equation*}
        \textit{CP}(h,r,t): h \xrightarrow[]{r'_1} \xrightarrow[]{r'_2} \xrightarrow[]{r'_3} ... \xrightarrow[]{r'_{n-1}} \xrightarrow[]{r'_n} t.
    \end{equation*}
    Such chains of relations, denoted as $R_{\textit{CP}}(r_q)$ are considered to be logical w.r.t $r_q$, which is defined as follows:
    \begin{equation*}
        R_{\textit{CP}}(r_q)=\{(r'_{1},r'_{2},...,r'_{n})\}.
    \end{equation*}
    Starting from $h_q$, path-based methods traverse the chain of relations in $R_{\textit{CP}}(r_q)$ on the testing graph and consider the ending entities as potential prediction candidates.

    CPs offer crucial connectivity information between heads and tails. However, most existing path-based approaches only rely on CPs for reasoning, which greatly hinders their reasoning ability. In this paper, we propose a more generalized type of path to provide logical supporting evidence w.r.t the query triplets, alleviating the reliance of CPs.

\section{Methodology}\label{section_methodology}
    In this section, we first introduce our key innovation, i.e., Anchoring Path (AP). We then define AP recall and AP accuracy for filtering out logically irrelevant APs. Furthermore, to fully leverage the knowledge provided by APs, we adopt the SOTA Sentence Transformer model and propose a unified architecture to retrieve detailed descriptions, encode paths into sentences, and make predictions.

    \subsection{Anchoring Path}
        \begin{myDef}[Anchoring Path (AP)]
        \label{def_Anchoring_path}
        Given a KG $G=\{(h_i,r_i,t_i)|i=1,2,3,..,m\}$ and a query triplet $(h,r,t)$, the head-Anchoring Path (head-AP) and tail-Anchoring Path (tail-AP) are defined as follows:
        \begin{equation*}
            \textit{AP}_{\textit{head}}(h,r,t): h \xrightarrow[]{r_1} \xrightarrow[]{r_2} ... \xrightarrow[]{r_{n-1}} \xrightarrow[]{r_n} e_n,
        \end{equation*}
        \begin{equation*}
            \textit{AP}_{\textit{tail}}(h,r,t): e'_0 \xrightarrow[]{r'_1}\xrightarrow[]{r'_2} ... \xrightarrow[]{r'_{n-1}} \xrightarrow[]{r'_n} t.
        \end{equation*}
        \end{myDef}
        $\textit{AP}_{\textit{head}}$ ($\textit{AP}_{\textit{tail}}$) focuses on the head (tail) entity and traverse the KG to find paths connecting to it. The afore-introduced Path 1 is a typical head-AP w.r.t triplet $(\textit{Morgan},\textit{Profession},\textit{Actor})$. In Theorems~\ref{theo_cp_to_ap} and~\ref{theo_any_path}, we discuss the relationship among CPs, APs, and any paths in KG (see the proofs of Theorems~\ref{theo_cp_to_ap} and~\ref{theo_any_path} in Appendix~\ref{sec_appendix_proof}).

        \begin{myTheo}
        \label{theo_cp_to_ap}
        Let $S(p)$ denote the set of all $p$ paths. For any triplet $(h,r,t)$,
        $S(\textit{CP}(h,r,t)) = S(\textit{AP}_{\textit{head}}(h,r,t)) \cap S(\textit{AP}_{\textit{tail}}(h,r,t))$.
        \end{myTheo}



        Theorem~\ref{theo_cp_to_ap} illustrates the relationship between APs and CPs. Specifically, the complete set of CPs is precisely the intersection of both head-APs and tail-APs. APs capture a wider range of diverse paths including all CPs. Such paths contain rich information due to their connections with head or tail entities, particularly useful when CPs are absent.

        To effectively leverage all relation information in KGs, inverse relations $R^{-1}=\{r^{-1}|r\in R\}$ and inverse triplets $\{(t_i,r_i^{-1},h_i)|i=1,2,3,...,m\}$ are usually added into the graph. Consequently, paths in the following forms:
        \begin{equation*}
        \begin{aligned}
            e_0 \xrightarrow[]{r_1}\xrightarrow[]{r_2} ... \xrightarrow[]{r_{n-1}} \xrightarrow[]{r_n} h, \\
            t \xrightarrow[]{r'_1}\xrightarrow[]{r'_2} ... \xrightarrow[]{r'_{n-1}}\xrightarrow[]{r'_n} e'_n,
        \end{aligned}
        \end{equation*}
        can be converted to $\textit{AP}_{\textit{head}}$ and $\textit{AP}_{\textit{tail}}$, respectively:
        \begin{equation*}
        \begin{aligned}
            h \xrightarrow[]{{r_{n-1}}^{-1}}\xrightarrow[]{{r_{n-2}}^{-1}} ... \xrightarrow[]{} \xrightarrow[]{{r_1}^{-1}} e_0, \\
            e'_n \xrightarrow[]{{r'_{n-1}}^{-1}}\xrightarrow[]{{r'_{n-2}}^{-1}} ... \xrightarrow[]{} \xrightarrow[]{{r'_1}^{-1}} t.
        \end{aligned}
        \end{equation*}

        By incorporating these inverse relations and triplets,  paths in $G'(E, R \cup R^{-1})$ could be simply categorized and represented (see Theorem~\ref{theo_any_path}).
        \begin{myTheo}
        \label{theo_any_path}
        Given a KG with inverse relations $G'(E, R \cup R^{-1})$ and a query triplet $(h,r,t)$, any path in $G'$ can be represented in one of the following exclusive categories:

            (\romannumeral1)~A path not traversing through either $h$ or $t$;

            (\romannumeral2)~An AP: $\textit{AP}_{\textit{head}}(h,r,t)$ or $\textit{AP}_{\textit{tail}}(h,r,t)$;

            (\romannumeral3)~A concatenation of APs: $\textit{AP}_1 \circ \textit{AP}_2 \circ ... \circ \textit{AP}_n$.
        \end{myTheo}



        Theorem~\ref{theo_any_path} demonstrates that any path having connections with the query triplet can be represented using an AP (Type~ \romannumeral2) or a concatenated AP (Type~ \romannumeral3). Empirically, paths that do not traverse through either $h$ or $t$ (Type~ \romannumeral1) are unlikely to provide useful information supporting the reasoning.
        To capture supporting evidence from both APs (Type~\romannumeral2) and concatenated APs (Type~\romannumeral3), we carefully design the architecture of APST, enabling it to take multiple paths (Types \romannumeral2~and~\romannumeral3) as inputs. For a concatenated AP, we may either regard it as a single path or break it into multiple APs to input together. Based on the preliminary results, we find that breaking concatenated APs into multiple APs usually leads to better performance in relation prediction. Therefore, for all experiments on APST in this paper, we break all concatenated APs and use the resulting APs as inputs.

    \subsection{Anchoring Path Filtering}
        APs can be extracted with ease by fixating the head or tail entity and then performing a random walk in the given KG. However, APs without restrictions, typically do not contribute logical supporting evidence. For instance, referring to Figure~\ref{fig1}, the following two APs do not provide logical evidence w.r.t triplet $(\textit{Morgan},\textit{Profession},\textit{Actor})$:
        \begin{equation*}
        \small
            \textit{Morgan} \xrightarrow[]{\textit{BornIn}} \textit{Memphis}\xrightarrow[]{\textit{State}} \textit{Tennessee},
        \end{equation*}
        \begin{equation*}
        \small
            \textit{Morgan} \xrightarrow[]{\textit{Gender}} \textit{Male}.
        \end{equation*}
        Such APs constitute the majority of randomly generated APs in real-world KGs. To eliminate these logically irrelevant APs, we propose AP accuracy and AP recall to measure their  rationality w.r.t the given query triplet. For brevity, we only use head-APs to introduce the following techniques, because it is trivial to apply these techniques to tail-APs.

        AP accuracy and AP recall are designed to identify logical APs. Given a query triplet $(h_q,r_q,t_q)$ and a corresponding $\textit{AP}_{\textit{head}}(h_q,r_q,t_q)$, we focus on $r_q$ and the chain of relations $r_{\textit{AP}_{\textit{head}}}=(r_1,r_2,r_3,...,r_n)$ on the AP. By traversing the training graph $G_{\textit{train}}$, we can find multiple paths that satisfy the chain of relations $r_{\textit{AP}_{\textit{head}}}$:
        \begin{equation*}
            \begin{aligned}
                h^1 \xrightarrow[]{r_1^1} \xrightarrow[]{r_2^1} &... \xrightarrow[]{r_{{n_1}-1}^1} \xrightarrow[]{r_{n_1}^1} e_{n_1}^1, \\
                h^2 \xrightarrow[]{r_1^2} \xrightarrow[]{r_2^2} &... \xrightarrow[]{r_{{n_2}-1}^2} \xrightarrow[]{r_{n_2}^2} e_{n_2}^2, \\
                &... \\
                h^m \xrightarrow[]{r_1^m} \xrightarrow[]{r_2^m} &... \xrightarrow[]{r_{{n_m}-1}^m} \xrightarrow[]{r_{n_m}^m} e_{n_m}^m.
            \end{aligned}
        \end{equation*}
        Let $E^{\textit{head}}_{\textit{AP}}=\{h^1,h^2,...,h^m\}$ denote the set of all starting entities of such paths and $E^{\textit{head}}_{r_q}=\{h| (h,r_q,t) \in G_{\textit{train}}\}$ denote the head entities of all triplets having relation $r_q$. Then any entity $e$ on $G_{\textit{train}}$ could be classified into one of the following exclusive categories:

            (\romannumeral1)~Path-Triplet \textbf{(PT)}: $e \in E^{\textit{head}}_{\textit{AP}}, e \in E^{\textit{head}}_{r_q}$,

            (\romannumeral2)~Path-Only \textbf{(PO)}: $e \in E^{\textit{head}}_{\textit{AP}}, e \notin E^{\textit{head}}_{r_q}$,

            (\romannumeral3)~Triplet-Only \textbf{(TO)}: $e \notin E^{\textit{head}}_{\textit{AP}}, e \in E^{\textit{head}}_{r_q}$,

            (\romannumeral4)~No-Connection \textbf{(NC)}: $e \notin E^{\textit{head}}_{\textit{AP}}, e \notin E^{\textit{head}}_{r_q}$.

        Logical APs typically possess the following favorable characteristics on $G_{\textit{train}}$:

        \noindent \textbf{High accuracy:}  Logical APs are expected to provide supporting evidence w.r.t the query triplet. Thus, if such a logical AP exists in $G_{\textit{train}}$, there should be a high likelihood of the query triplet's existence. To measure the existence ratio of the query triplet, we define AP accuracy as follows:
        \begin{equation}\label{eq_acc}
            \textit{acc}(r_q, r_{\textit{AP}_{\textit{head}}})= \frac{\#\textit{PT}}{\#\textit{PO}+\#\textit{PT}},
        \end{equation}
        where \# denotes the number of all entities in the respective categories.

        \noindent \textbf{High recall:} Certain APs with high accuracy may not contribute much towards relation predictions because these APs rarely appear in $G_{\textit{test}}$. Empirically speaking, such APs usually also appear infrequently in $G_{\textit{train}}$. To better preserve the generalization from $G_{\textit{train}}$ to $G_{\textit{test}}$, logical APs corresponding to the query triplet should have high occurrence in $G_{\textit{train}}$. To measure the occurrence ratio of an AP, we define AP recall as follows:
        \begin{equation}\label{eq_rec}
            \textit{rec}(r_q, r_{\textit{AP}_{\textit{head}}})= \frac{\#\textit{PT}}{\#\textit{TO}+\#\textit{PT}}.
        \end{equation}

    With AP accuracy and AP recall defined, the rationality of $r_{\textit{AP}_{\textit{head}}}$ corresponding to $r_q$ can be quantified.

    During training, we extract logical APs as follows: (\romannumeral1)~Choose a relation $r_q$ and extract a set of potential APs corresponding to $r_q$ from the given KG; (\romannumeral2)~Evaluate the extracted APs using~(\ref{eq_acc}) and/or~(\ref{eq_rec}); (\romannumeral3)~Exclude APs falling below a pre-determined threshold for either of the two metrics from the potential set; (\romannumeral4)~Store remaining logical APs for $r_q$; and (\romannumeral5)~Choose another relation $r'_q$ and repeat Steps~(\romannumeral1) to~(\romannumeral4) until APs for all relations are extracted.

    During testing, we can find many APs corresponding to the testing query triplets. However, only APs that match the chain of relations in the stored logical AP sets are extracted and utilized as supporting evidence.

    \subsection{Anchoring Path Sentence Transformer}

    \begin{figure*}[!t]
        \centering
        \includegraphics[scale=0.36]{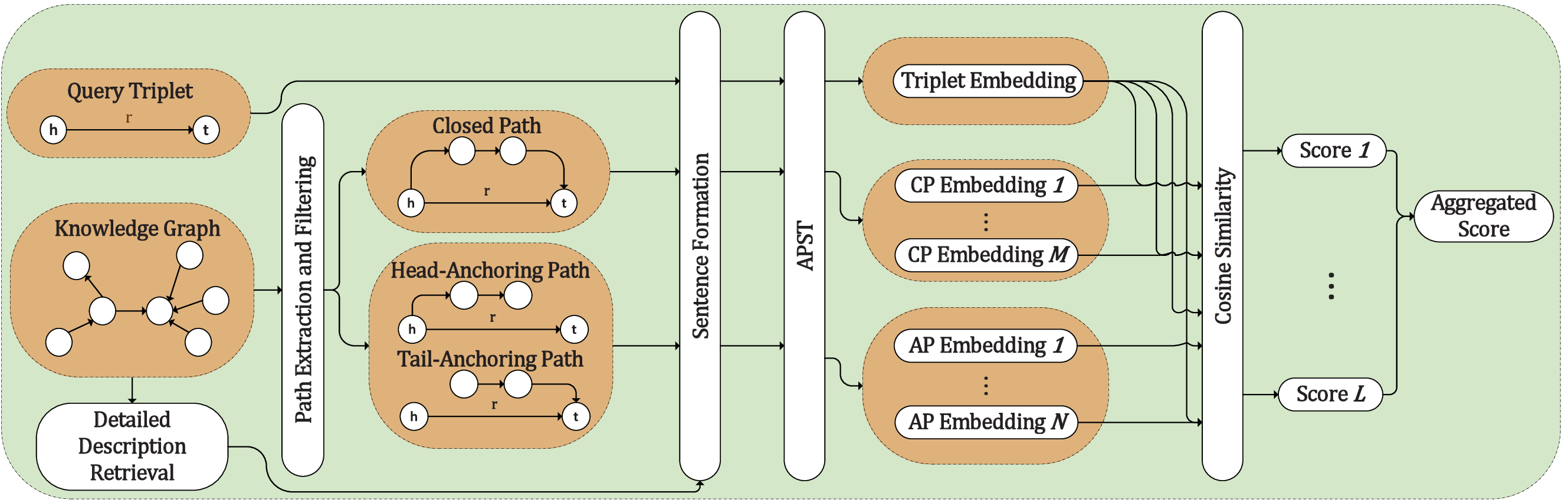}
        \caption{Architecture of Anchoring Path Sentence Transformer (APST).}\label{fig2}
    \end{figure*}

    To fully utilize the supporting evidence provided by the logical APs and incorporate text descriptions for more accurate relation prediction, we propose Anchoring Path Sentence Transformer (APST) by extending the SOTA relation prediction approach~\cite{KRST}. Figure~\ref{fig2} shows the architecture of APST. Specifically, APST sequentially processes the following steps: (\romannumeral1)~Detailed description retrieval; (\romannumeral2)~AP extraction and filtering;  (\romannumeral3)~Sentence formation and AP encoding; and (\romannumeral4)~APST prediction.

    \noindent \textbf{Detailed description retrieval:}  To effectively utilize the semantic knowledge in PLMs, rich text descriptions of entities and relations are necessary. However, descriptions are not mandatory when constructing KGs, hence are not provided in many KGs. Even the commonly adopted KG datasets, such as FB15k-237 and NELL-995, only provide the entity names with a couple of words as descriptions, making it difficult for the Sentence Transformer to accurately capture the semantic meaning of entities. In APST, we deploy Google Custom Search API
    to search on specialized websites and automatically retrieve detailed descriptions. Google Custom Search allows us to search on a specific website and return pages relevant to the given keywords. Such descriptions on the returned pages are detailed and can be formulated into input sentences.

    \noindent \textbf{AP extraction and filtering:} As discussed in the context of Theorem~\ref{theo_cp_to_ap}, CPs are subsets of APs. Because CPs connect to both the head and tail entities, they are usually more informative than other APs. Therefore, APST first generates APs and filters out logically irrelevant ones using~(\ref{eq_acc}) and/or~(\ref{eq_rec}) as previously introduced. Subsequently, a number of $L$ paths are chosen as inputs for APST, including all CPs and several randomly selected logical APs.

    \noindent \textbf{Sentence formation and AP encoding:} To take advantage of the text descriptions for entities and relations, we convert the query triplets and the filtered APs into sentences.  With detailed descriptions for entities, our input sentences for CPs, APs, and query triplets are formulated as follows:
    \begin{equation*}
    \begin{aligned}
        \textit{Sentence}_{\textit{CP}}: \: & [e_1][D_{e_1}][r_1][e_2][r_2]...[e_{n-1}][r_n][e_n][D_{e_n}], \\
        \textit{Sentence}_{\textit{AP}}: \: & [e_1][D_{e_1}][r_1][r_2]...[r_n][e_n][D_{e_n}], \\
        \textit{Sentence}_{q}: \: & [h_q][D_{h_q}][r_q][t_q][D_{t_q}],
    \end{aligned}
    \end{equation*}
    where $D_e$ denotes the detailed descriptions for the corresponding entity $e$.
    We only use detailed descriptions for the first and last entities because the descriptions for these two entities are usually the most important ones. Moreover, including detailed descriptions for all entities often results in surpassing the maximum input length permitted by the adopted Sentence Transformer, leading to the unwanted truncation of the overall description.
    Then we apply Sentence Transformer to encode paths and triplets. Specifically, the formed sentences are encoded into embeddings. These embeddings contain textual information of triplets and APs, which can be used for the next APST prediction step.

    \noindent \textbf{APST prediction:} During training, triplets with labels 1 and -1 are regarded as positive and negative triplets, respectively. In APST, APs are considered as having similar semantic meanings with the corresponding positive triplet. A query triplet is predicted as positive if at least one of the corresponding APs has a high similarity score with it. To measure the similarity score between an AP and the corresponding query triplet, we apply cosine similarity as follows:
    \begin{equation}
        \textit{Sim}(s_{\textit{hrt}},s_\textit{AP})= \frac{s_{\textit{hrt}} s_\textit{AP}}{|s_{\textit{hrt}}||s_\textit{AP}|},
    \end{equation}
    where $s_{\textit{hrt}}$ and $s_\textit{AP}$ denote the sentence embeddings of the query triplet and the corresponding AP, respectively.
    APST allows up to a number of $L$ APs as inputs to compute the similarity scores. The AP with the highest similarity score is considered to be the most reasonable path and the highest similarity score is considered to be the final score for the corresponding triplet:
    \begin{equation}
        \textit{Score}(h,r,t)=max\{\textit{Sim}(s_{\textit{hrt}},s_{\textit{AP}_i})|i=1,2,...,L\}.
    \end{equation}
    During training, we use cosine embedding loss as the training  signal, which is formulated as follows:
    \begin{equation}
    \label{eq_cos_loss}
        \mathscr{L}(h,r,t,y)=\left\{
        \begin{aligned}
        & 1-\textit{Score}(h,r,t), &  y=1,\\
        & \max(0,\textit{Score}(h,r,t)-M), & y=-1,
        \end{aligned}
        \right.
    \end{equation}
    where $M \in (-1,1)$ is a pre-defined margin and $y \in \{1,-1\}$ denotes the label.

    \begin{table*}[tbh!]
    \footnotesize
    \renewcommand\arraystretch{0.75}
    \centering
    \caption{Transductive and inductive results on WN18RR, FB15k-237 and NELL-995}
    \label{table_transductive_and_inductive}
    \begin{tabular}{llcccccc}\toprule
          &         & \multicolumn{3}{c}{Transductive}                 & \multicolumn{3}{c}{Inductive}                    \\
           \cmidrule(lr){3-5} \cmidrule(lr){6-8}
          &         & WN18RR         & FB15k-237      & NELL-995       & WN18RR         & FB15k-237      & NELL-995       \\ \midrule \midrule
    MRR   & RuleN   & 0.669          & 0.674          & 0.736          & 0.780          & 0.462          & 0.710          \\
          & GRAIL   & 0.676          & 0.597          & 0.727          & 0.799          & 0.469          & 0.675          \\
          & MINERVA & 0.656          & 0.572          & 0.592          & -              & -              & -              \\
          & TuckER  & 0.646          & 0.682          & \underline{0.800}    & -              & -              & -              \\
          & KG-BERT & -              & -              & -              & 0.547          & 0.500          & 0.419          \\
          & BERTRL  & 0.683          & 0.695          & 0.781          & 0.792          & 0.605          & \textbf{0.808} \\
          & KRST    & \underline{0.899}    & \underline{0.720}    & \underline{0.800}    & \underline{0.890}    & \underline{0.716}    & \underline{0.769}    \\  \cmidrule(lr){2-8}
          & APST    & \textbf{0.902} & \textbf{0.774} & \textbf{0.801} & \textbf{0.908} & \textbf{0.764} & \underline{0.769}    \\ \midrule \midrule
    Hit@1 & RuleN   & 0.646          & 0.603          & 0.636          & 0.745          & 0.415          & 0.638          \\
          & GRAIL   & 0.644          & 0.494          & 0.615          & 0.769          & 0.390          & 0.554          \\
          & MINERVA & 0.632          & 0.534          & 0.553          & -              & -              & -              \\
          & TuckER  & 0.600          & 0.615          & \textbf{0.729} & -              & -              & -              \\
          & KG-BERT & -              & -              & -              & 0.436          & 0.341          & 0.244          \\
          & BERTRL  & 0.655          & 0.620          & 0.686          & 0.755          & 0.541          & \textbf{0.715} \\
          & KRST    & \underline{0.835}    & \underline{0.639}    & 0.694          & \underline{0.809}    & \underline{0.600}    & 0.649          \\  \cmidrule(lr){2-8}
          & APST    & \textbf{0.839} & \textbf{0.694} & \underline{0.698}    & \textbf{0.837} & \textbf{0.643} & \underline{0.663}   \\ \bottomrule
    \end{tabular}
    \end{table*}

\section{Experimental Results}\label{section_experiments}
    We conduct extensive experiments using three datasets to comprehensively evaluate the performance of APST. Additionally, we perform an ablation study to assess the impact and effectiveness of APs and detailed descriptions. To further illustrate their significance, we present a case study showcasing AP explanations on FB15k-237 in Appendix~\ref{sec_appendix_case_study}.

    \subsection{Experimental Setup and Implementation Details}
    Following the setup used in \cite{BERTRL} and \cite{KRST}, we conduct experiments using the commonly benchmarked transductive and inductive datasets introduced by \cite{GRAIL}, which are the subsets of  WN18RR, FB15k-237, and NELL-995 (see Appendix~\ref{sec_appendix_dataset}). Moreover, for each query triplet, 1 positive and 49 negative candidate entities are provided by \cite{BERTRL} for evaluation. We measure the mean reciprocal rank (MRR) and hit rate (Hit@1) among the 50 candidate entities. For inductive experiments, there is zero overlap between training and testing entities. Few-shot experiments are conducted on both inductive and transductive datasets, with 1000 or 2000 training triplets randomly selected.

    We implement APST\footnote{github.com/ZhixiangSu/APST} based on the SOTA Sentence Transformer (all-mpnet-base-v2\footnote{huggingface.co/sentence-tformers/all-mpnet-base-v2}) using PyTorch and train it on two NVIDIA Tesla V100 GPUs with 32GB RAM.

    In APST, we apply both~(\ref{eq_acc}) and~(\ref{eq_rec}) with thresholds for filtering. More detailed hyper-parameter settings are introduced in Appendix~\ref{sec_appendix_hyper}. To automatically retrieve detailed descriptions, we use the entity names along with short descriptions in all datasets as keywords for Google Custom Search on entity deciphering websites. Because FB15k-237 and NELL-995 only contain commonsense knowledge, we search their entities on Wikipedia.
    For multilingual words in WN18RR, we search them on Wiktionary.

    \begin{table*}[tbh!]
    \footnotesize
    \renewcommand\arraystretch{0.75}
    \centering
    \caption{Few-shot results on WN18RR, FB15k-237 and NELL-995}
    \label{table_few_shot}
    \begin{tabular}{llcccccccccccc}
                \toprule
              &         & \multicolumn{6}{c}{Transductive}                                                                    & \multicolumn{6}{c}{Inductive}                                                                       \\ \cmidrule(lr){3-8} \cmidrule(lr){9-14}
              &         & \multicolumn{2}{c}{WN18RR}      & \multicolumn{2}{c}{FB15k-237}   & \multicolumn{2}{c}{NELL-995}    & \multicolumn{2}{c}{WN18RR}      & \multicolumn{2}{c}{FB15k-237}   & \multicolumn{2}{c}{NELL-995}    \\\cmidrule(lr){3-4} \cmidrule(lr){5-6} \cmidrule(lr){7-8} \cmidrule(lr){9-10} \cmidrule(lr){11-12} \cmidrule(lr){13-14}
              &         & 1000           & 2000           & 1000           & 2000           & 1000           & 2000           & 1000           & 2000           & 1000           & 2000           & 1000           & 2000           \\\midrule \midrule
        MRR   & RuleN   & 0.567          & 0.625          & 0.434          & 0.577          & 0.453          & 0.609          & 0.681          & 0.773          & 0.236          & 0.383          & 0.334          & 0.495          \\
              & GRAIL   & 0.588          & 0.673          & 0.375          & 0.453          & 0.292          & 0.436          & 0.652          & 0.799          & 0.380          & 0.432          & 0.458          & 0.462          \\
              & MINERVA & 0.125          & 0.268          & 0.198          & 0.364          & 0.182          & 0.322          & -              & -              & -              & -              & -              & -              \\
              & TuckER  & 0.258          & 0.448          & 0.457          & 0.601          & 0.436          & 0.577          & -              & -              & -              & -              & -              & -              \\
              & KG-BERT & -              & -              & -              & -              & -              & -              & 0.471          & 0.525          & 0.431          & 0.460          & 0.406          & 0.406          \\
              & BERTRL  & 0.662          & 0.673          & 0.618          & 0.667          & 0.648          & 0.693          & 0.765          & 0.777          & 0.526          & 0.565          & 0.736          & \underline{0.744}    \\
              & KRST    & \underline{0.871}    & \textbf{0.882} & \underline{0.696}    & \underline{0.701}    & \underline{0.743}    & \textbf{0.781} & \underline{0.886}    & \underline{0.878}    & \underline{0.679}    & \underline{0.680}    & \underline{0.745}    & 0.738          \\ \cmidrule(lr){2-14}
              & APST    & \textbf{0.874} & \underline{0.880}    & \textbf{0.724} & \textbf{0.753} & \textbf{0.745} & \underline{0.767}    & \textbf{0.894} & \textbf{0.879} & \textbf{0.697} & \textbf{0.747} & \textbf{0.765} & \textbf{0.747} \\ \midrule \midrule
        Hit@1 & RuleN   & 0.548          & 0.605          & 0.374          & 0.508          & 0.365          & 0.501          & 0.649          & 0.737          & 0.207          & 0.344          & 0.282          & 0.418          \\
              & GRAIL   & 0.489          & 0.633          & 0.267          & 0.352          & 0.198          & 0.342          & 0.516          & 0.769          & 0.273          & 0.351          & 0.295          & 0.298          \\
              & MINERVA & 0.106          & 0.248          & 0.170          & 0.324          & 0.152          & 0.284          & -              & -              & -              & -              & -              & -              \\
              & TuckER  & 0.230          & 0.415          & 0.407          & 0.529          & 0.392          & 0.520          & -              & -              & -              & -              & -              & -              \\
              & KG-BERT & -              & -              & -              & -              & -              & -              & 0.364          & 0.404          & 0.288          & 0.317          & 0.236          & 0.236          \\
              & BERTRL  & 0.621          & 0.637          & 0.517          & 0.583          & 0.526          & 0.582          & 0.713          & 0.731          & 0.441          & 0.493          & 0.622          & 0.628          \\
              & KRST    & \underline{0.790}    & \underline{0.810}    & \underline{0.611}    & \underline{0.602}    & \underline{0.628}    & \textbf{0.678} & \underline{0.811}    & \underline{0.793}    & \underline{0.537}    & \underline{0.524}    & \underline{0.637}    & \underline{0.629}    \\\cmidrule(lr){2-14}
              & APST    & \textbf{0.798} & \textbf{0.813} & \textbf{0.632} & \textbf{0.665} & \textbf{0.640} & \underline{0.663}    & \textbf{0.822} & \textbf{0.798} & \textbf{0.561} & \textbf{0.627} & \textbf{0.654} & \textbf{0.637} \\\bottomrule
        \end{tabular}
        \end{table*}

    \begin{figure*}[!t]
        \centering
        \begin{minipage}{500pt}
            \centering
            \emph{\textbf{MRR}}
        \end{minipage}
        \subfigure[FB15k-237-t]{
            \includegraphics[scale=0.30]{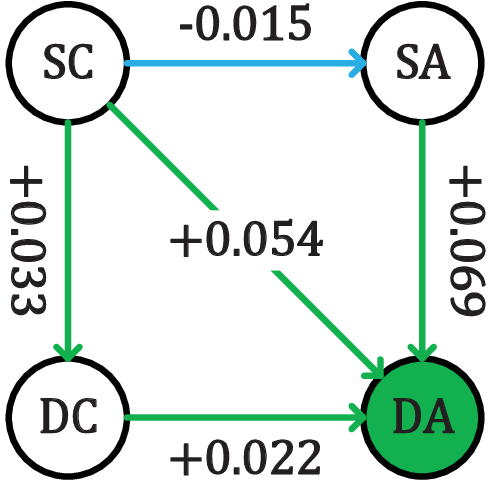}
            \label{fig3.1.1}
            }
        \subfigure[FB15k-237-i]{
            \includegraphics[scale=0.30]{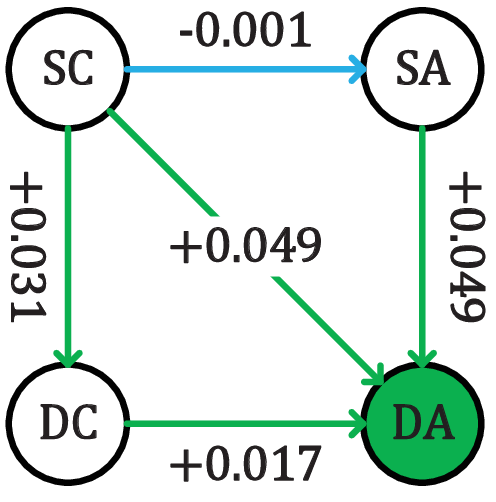}
            \label{fig3.2.1}
            }
        \subfigure[NELL-995-t]{
            \includegraphics[scale=0.30]{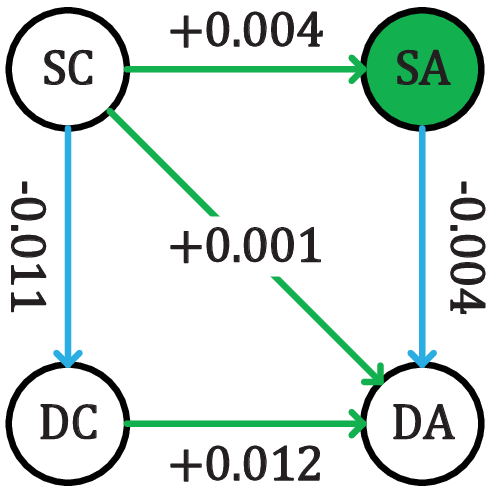}
            \label{fig3.3.1}
            }
        \subfigure[NELL-995-i]{
            \includegraphics[scale=0.30]{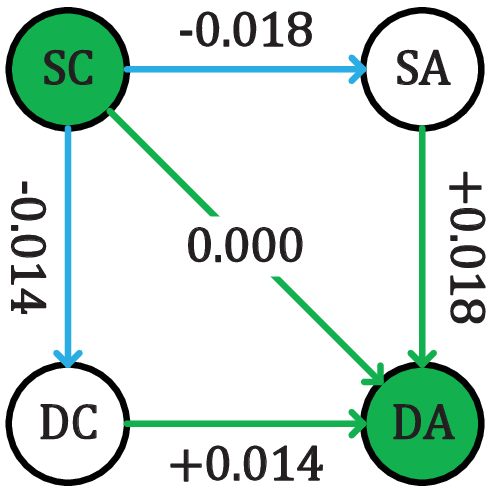}
            \label{fig3.4.1}
            }
        \subfigure[WN18RR-t]{
            \includegraphics[scale=0.30]{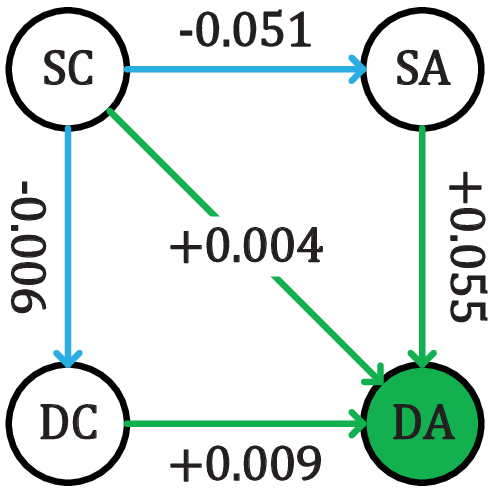}
            \label{fig3.5.1}
            }
        \subfigure[WN18RR-i]{
            \includegraphics[scale=0.30]{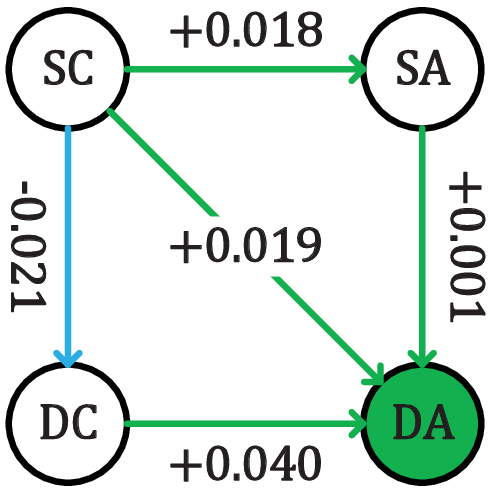}
            \label{fig3.6.1}
            }
            \\
        \begin{minipage}{500pt}
            \centering
            \emph{\\ \textbf{Hit@1}}
        \end{minipage}
        \subfigure[FB15k-237-t]{
            \includegraphics[scale=0.30]{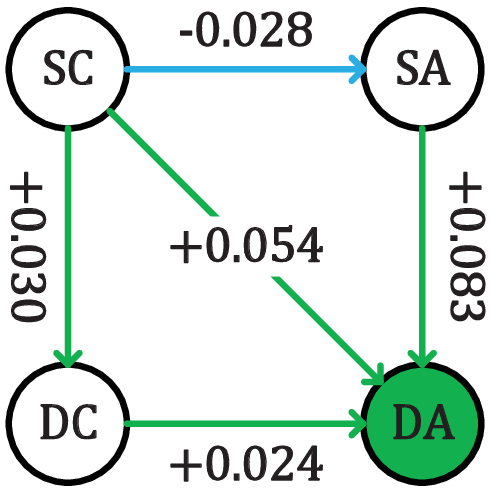}
            \label{fig3.1.2}
            }
        \subfigure[FB15k-237-i]{
            \includegraphics[scale=0.30]{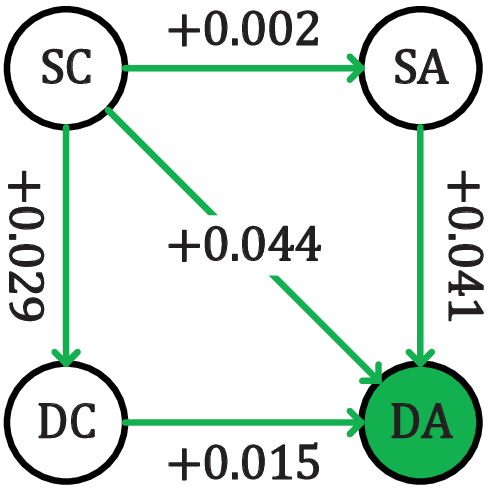}
            \label{fig3.2.2}
            }
        \subfigure[NELL-995-t]{
            \includegraphics[scale=0.30]{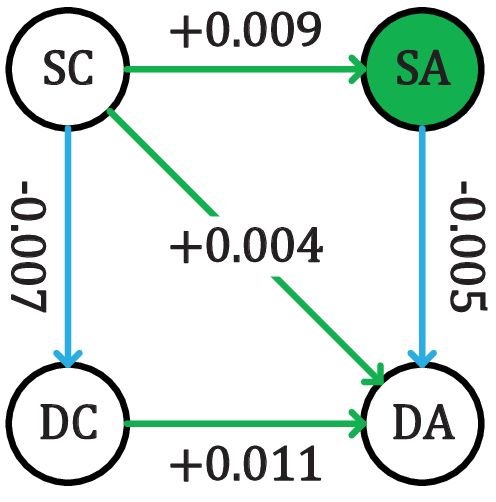}
            \label{fig3.3.2}
            }
        \subfigure[NELL-995-i]{
            \includegraphics[scale=0.30]{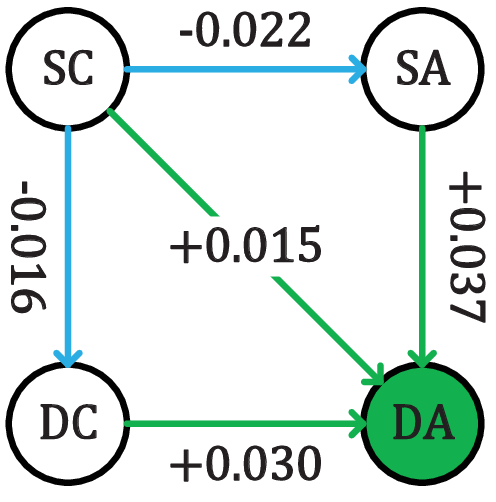}
            \label{fig3.4.2}
            }
        \subfigure[WN18RR-t]{
            \includegraphics[scale=0.30]{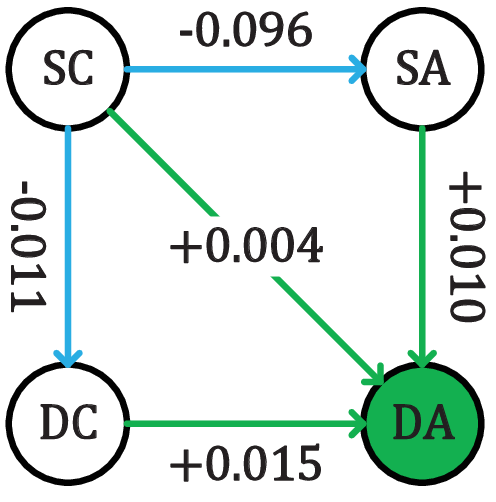}
            \label{fig3.5.2}
            }
        \subfigure[WN18RR-i]{
            \includegraphics[scale=0.30]{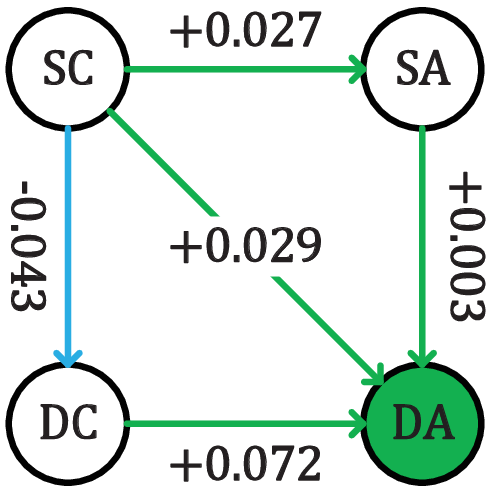}
            \label{fig3.6.2}
            }
        \caption{Ablation study results on FB15k-237, NELL-995 and WN19RR (-t: transductive;
            -i: inductive).}
        \label{ablation_study_fig}
    \end{figure*}


    \subsection{Transductive and Inductive Relation Predictions}

    We compare APST against various approaches including RuleN, GRAIL, MINERVA, TuckER, KG-BERT,  BERTRL, and KRST. Among these baselines, TuckER, GRAIL, and KRST are the SOTA embedding-based, GNN-based, and path-based methods, respectively.

    For both transductive and inductive settings,
    training graph $G_{\textit{train}}$ is used to extract logical APs for training APST. However, during testing in inductive settings, APs for query triplets do not exist in $G_{\textit{train}}$ due to zero overlap between $E_{\textit{train}}$ and $E_{\textit{query}}$  (see Definition~\ref{def_inductive_relation_prediction}). Therefore, we follow the configuration applied in \cite{KRST} and extract APs from $G_{\textit{test}}$ in inductive settings.

    Table~\ref{table_transductive_and_inductive} reports the results of transductive and inductive relation predictions. APST achieves the best (bolded) MRR performance in 5 out of 6 settings and the best Hit@1 rate in 4 out of 6. For the remaining settings, APST achieves the second-best (underlined) performance. Notably, APST's performance is significantly better in the transductive setting of FB15k-237 (+0.054 for MRR and +5.5\% for Hit@1), inductive setting of WN18RR (+0.018 for MRR and +2.8\% for Hit@1), and inductive setting of FB15k-237 (+0.048 for MRR and +4.3\% for Hit@1). In both the transductive and inductive settings of NELL-995, APST outperforms KRST. However, APST only achieves the second-best performance on NELL-995, we hold the view that APST's performance is limited by the retrieved detailed descriptions (see ablation study results for more details).

    \subsection{Few-shot Relation Predictions}

    We use the subsets, containing 1000 (or 2000) training triplets for all three datasets provided by ~\cite{BERTRL}, to further benchmark the performance of APST in the few-shot experiments.
    As shown in Table~\ref{table_few_shot}, APST demonstrates superior performance in majority few-shot scenarios (10 out of 12 for MRR and 11 out of 12 for Hit@1). For the remaining scenarios, APST achieves the second-best performance. These results underscore the robust transferability of APST, showcasing its remarkable performance even with a limited number of training pairs.

    \subsection{Ablation Study Results}
    To further investigate the effectiveness of the key components in APST, we conduct ablation studies using all three datasets in both the transductive and inductive settings.  The various testing scenarios are listed as follows:

        (\romannumeral1)~\textbf{SC:} Short descriptions with CPs only,

        (\romannumeral2)~\textbf{SA:} Short descriptions with APs,

        (\romannumeral3)~\textbf{DC:} Detailed descriptions with CPs only,

        (\romannumeral4)~\textbf{DA:} Detailed descriptions with APs.

    Specifically, for SC and SA, we remove the retrieved detailed descriptions and only apply the short descriptions provided in the respective datasets. For SC and DC, we only use CPs and remove all other APs for encoding and prediction. The results of the ablation studies are visualized in Figure~\ref{ablation_study_fig}.

    APST takes in inputs of APs with detailed descriptions, i.e., DA. DA achieves the best performance in 5 out of 6 cases in terms of MRR and Hit@1, respectively. This finding demonstrates that both long descriptions and APs contribute to elevated performance. DA leads to a better performance than DC in all settings, suggesting that with long descriptions, APs always lead to performance gain. This phenomenon does not apply between SA and SC. Thus, we hold the view that detailed descriptions are significantly helpful when APs are used for predictions. Comparing SC with DC and comparing SA with DA, results show that detailed descriptions do not always lead to better performance (5 of 12 for both MRR and Hit@1). Probably the most plausible reason is that our detailed descriptions are obtained from Wikipedia and Wiktionary. FB15k-237 originates from Wikipedia, and its entities exactly match the objects found in Wikipedia. However, for entities in NELL-995 and WN18RR, their meanings may not exactly match the descriptions, which leads to a performance decrease. In summary, although using detailed descriptions may negatively impact the effectiveness of CPs, utilizing APs with detailed descriptions will result in a notable enhancement and yield superior performance.

\section{Conclusion}
    In this paper, we propose APST, which is a novel path-based method utilizing Sentence Transformer. APST retrieves detailed descriptions, captures supporting evidence in the form of Anchoring Paths (APs), and incorporates both AP and CP inputs for comprehensive predictions and convincing explanations. Through experiments on commonly applied datasets, APST demonstrates the potential for leveraging textual knowledge from external sources to enrich KGs and take advantage of APs to improve prediction accuracy.

    Nonetheless, the current APST model primarily focuses on relatively small datasets. Going forward, we aim to develop a more memory-efficient variant to handle increased data complexity and apply APST to larger datasets.

\section{Acknowledgements}
    This research is supported, in part, by the National Research Foundation, Prime Minister’s Office, Singapore under its NRF Investigatorship Programme (NRFI Award No. NRF-NRFI05-2019-0002). Any opinions, findings and conclusions or recommendations expressed in this material are those of the author(s) and do not reflect the views of National Research Foundation, Singapore. This research is also supported by the National Key R\&D Program of China No. 2021YFF0900800 and Shandong Provincial Key Research and Development Program (Major Scientific and Technological Innovation Project) (No. 2021CXGC010108).


\section{Appendices}
\setcounter{secnumdepth}{2}

\begin{appendix}

\begin{table*}[!t]
        \footnotesize
        \caption{Selected examples of APs extracted from FB15k-237}
        \label{table_case_study}
        \begin{tabular}{llcc}
        \toprule
        Triplet                                                         & Anchoring Path                                                                                  & Accuracy & Recall \\ \midrule \midrule
        $h \xrightarrow[]{country/official\_language}t$          & $e_0 \xrightarrow[]{film/country}\xrightarrow[]{{human\_language/countries\_spoken\_in}^{-1}}t$ & 0.67     & 1.00   \\
        $h \xrightarrow[]{bibs\_location/country} t$             & $e_0 \xrightarrow[]{{people/place\_lived/location}^{-1}}\xrightarrow[]{person/nationality}t$    & 1.00     & 0.67   \\
        $h \xrightarrow[]{tv\_program/languages} t$              & $e_0 \xrightarrow[]{tv/tv\_producer\_term/program}\xrightarrow[]{film/language}t$               & 1.00     & 1.00   \\
        $h \xrightarrow[]{netflix\_genre/titles} t$              & $h\xrightarrow[]{{film/genre}^{-1}}\xrightarrow[]{film\_crew\_gig/film\_crew\_role}e_n$         & 0.60     & 0.67   \\
        $h \xrightarrow[]{film/edited\_by} t$                    & $h\xrightarrow[]{{director/film}^{-1}}\xrightarrow[]{{film/edited\_by}^{-1}}e_n$                & 1.00     & 1.00   \\
        $h \xrightarrow[]{ethnicity/geographic\_distribution} t$ & $h\xrightarrow[]{{film/country}^{-1}}\xrightarrow[]{film/language}e_n$                          & 0.67     & 1.00   \\ \bottomrule
        \end{tabular}
    \end{table*}

    \begin{table}[!t]
    \footnotesize
    \renewcommand\arraystretch{0.8}
        \centering
        \caption{Statistics of datasets}\label{table_dataset}
        \begin{tabular}{lllll}
        \toprule
        Dataset   & $G$               & $|R_G|$       & $|E_G|$      & \#Triplets \\ \midrule \midrule
        WN18RR    & train             & 9           & 2746       & 6670       \\
                  & train-2000        & 9           & 1970       & 2002       \\
                  & train-1000        & 9           & 1362       & 1001       \\
                  & test-transductive & 7           & 962        & 638       \\
                  & test-inductive    & 8           & 922        & 1991       \\ \midrule
        FB15k-237 & train             & 180         & 1594       & 5223       \\
                  & train-2000        & 180         & 1280       & 2008       \\
                  & train-1000        & 180         & 923        & 1027       \\
                  & test-transductive & 102         & 550        & 492        \\
                  & test-inductive    & 142         & 1093       & 2404       \\
                  \midrule
        NELL-995  & train             & 88          & 2564       & 10063      \\
                  & train-2000        & 88          & 1346       & 2011       \\
                  & train-1000        & 88          & 893        & 1020       \\
                  & test-transductive & 60          & 1936       & 968        \\
                  & test-inductive    & 79          & 2086       & 6621       \\
                  \bottomrule
        \end{tabular}
    \end{table}

    \begin{table}[!t]
    \footnotesize
    \centering
    \caption{Hyper-parameters for transductive and inductive experiments}
    \label{table_hyper}
    \begin{tabular}{lc}
    \toprule
    Hyper-parameter                                                                                                  & Value     \\ \midrule \midrule
    Random seed                                                                                                   & 42        \\
    Samples for training and validation                                                                           & 5         \\
    Samples for testing                                                                                           & 50        \\
    \begin{tabular}[c]{@{}l@{}}Number of APs extracted for training,\\  validation and testing ($L$)\end{tabular} & 3         \\
    AP search depth                                                                                               & 2         \\
    AP accuracy threshold                                                                                         & 0.5       \\
    AP recall threshold                                                                                           & 0.5       \\
    Number of epochs                                                                                              & 30        \\
    Sentence transformer learning rate                                                                            & $10^{-5}$ \\ \bottomrule
    \end{tabular}
    \end{table}
    \section{Proof of Theorems}\label{sec_appendix_proof}
    \begin{myTheoApp}
        Let $S(p)$ denote the set of all $p$ paths. For any triplet $(h,r,t)$,
        $S(\textit{CP}(h,r,t)) = S(\textit{AP}_{\textit{head}}(h,r,t)) \cap S(\textit{AP}_{\textit{tail}}(h,r,t))$.
        \end{myTheoApp}
        \begin{proof}
            Consider an arbitrary path $p_1 \in S(\textit{CP}(h,r,t))$. Because $p_1$ starts from entity $h$ and ends at entity $t$, it naturally belongs to both $S(\textit{AP}_{\textit{head}}(h,r,t))$ and $S(\textit{AP}_{\textit{tail}}(h,r,t))$. Therefore, $p_1 \in S(\textit{AP}_{\textit{head}}(h,r,t)) \cap S(\textit{AP}_{\textit{tail}}(h,r,t))$, establishing that $S(\textit{CP}(h,r,t)) \subseteq S(\textit{AP}_{\textit{head}}(h,r,t)) \cap S(\textit{AP}_{\textit{tail}}(h,r,t))$.

            Let $p_2$ denote a path in $S(\textit{AP}_{\textit{head}}(h,r,t)) \cap S(\textit{AP}_{\textit{tail}}(h,r,t))$. This implies that $p_2$ starts from entity $h$ and ends at entity $t$ simultaneously. Consequently, $p_2$ also qualifies as a valid path in $S(\textit{CP}(h,r,t))$. Thus, $S(\textit{AP}_{\textit{head}}(h,r,t)) \cap S(\textit{AP}_{\textit{tail}}(h,r,t)) \subseteq S(\textit{CP}(h,r,t))$.

            Combining the two inclusions, we can conclude that $S(\textit{CP}(h,r,t)) = S(\textit{AP}_{\textit{head}}(h,r,t)) \cap S(\textit{AP}_{\textit{tail}}(h,r,t))$.
        \end{proof}

    \begin{myTheoApp}
        Given a KG with inverse relations $G'(E, R \cup R^{-1})$ and a query triplet $(h,r,t)$, any path in $G'$ can be represented in one of the following exclusive categories:

            (\romannumeral1) A path not traversing through either $h$ or $t$;

            (\romannumeral2) An AP: $\textit{AP}_{\textit{head}}(h,r,t)$ or $\textit{AP}_{\textit{tail}}(h,r,t)$;

            (\romannumeral3) A concatenation of APs: $\textit{AP}_1 \circ \textit{AP}_2 \circ ... \circ \textit{AP}_n$.
        \end{myTheoApp}
        \begin{proof}
            Given a path $p$ on $G'$, let $E_p=(e_0,e_1,...,e_n)$ denote the entity chain of $p$. For simplicity, let $I_{ht}=\{i_1,i_2,...,i_m | e_{i_m}=h$ or $e_{i_m} = t\}$ be the set of entity indexes that correspond to either $h$ or $t$ present in $E_p$. If $I_{ht}$ is an empty set ($|I_{ht}|=0$), $p$ can be classified as Type~\romannumeral1. If $I_{ht}$ only contains the first or last entity ($I_{ht}=\{0\}$ or $I_{ht}=\{n\}$), $p$ can be classified as Type~\romannumeral2.

            Otherwise, $p$ can be divided into multiple sub-paths based on the indexes in $I_{ht}$:
            1) Sub-paths of the form $\{(e_{i_a},e_{i_a+1},...,e_{i_b})|\forall i_k \in I_{ht}, i_k \leq i_a$ or $i_k \geq i_b\}$. Because $e_{i_a}=h$ or $t$, such sub-paths are APs.
            2) Sub-paths of the form $\{(e_0,e_{1},...,e_{i_b})|\forall i_k \in I_{ht}, i_k \geq i_b\}$ or $\{(e_{i_a},e_{i_a+1},...,e_n)|\forall i_k \in I_{ht}, i_k \leq i_a\}$. If $i_b=0$ or $i_a=n$, these particular sub-paths do not exist. Otherwise, because $i_b,i_a \in I_{ht}$, these sub-paths are APs.

            Therefore, in all given situations, any path on $G'$ can be represented in one of the specified categories,exclusively.
        \end{proof}

\section{Case Study}\label{sec_appendix_case_study}

    To demonstrate the APs generated by APST are logical, we conduct a case study on the FB15k-237 dataset. Table~\ref{table_case_study} presents six query triplets and their corresponding APs along with AP accuracy and AP recall.

    The triplet $(h, \textit{country/official\_language}, t)$  and its corresponding AP present a comprehensive explanation and highlight the difference between AP accuracy and AP recall. To determine the official language of a country, we can search for films and utilize those corresponding to country $h$ to identify potential languages. However, in many countries, films have various versions of languages. Therefore, the accuracy of the corresponding AP is not ensured (0.67). Nonetheless, the recall for the corresponding AP is 1.00 because there is consistently a version of film available in the official language. Thus, the positive tail entity $t$ is always one of the ending entities obtained by traversing the corresponding AP. Similar explanations can be given for other triplets and their corresponding APs.

\section{Datasets}\label{sec_appendix_dataset}
    The statistics of the datasets used in this work are shown in Table~\ref{table_dataset}.

\section{Hyper-parameter Values}\label{sec_appendix_hyper}
    The hyper-parameter values of APST  for the transductive and inductive experiments are shown in Table~\ref{table_hyper}.

\end{appendix}
\bibliography{ref.bib}
\end{document}